\documentclass{sig-alternate-2013}

\usepackage{subfigure}

\usepackage[ruled,vlined,linesnumbered,commentsnumbered]{algorithm2e}

\usepackage{url}
\usepackage{verbatim}
\usepackage{multirow}

\newcounter{def_count}
\newtheorem{definition}[def_count]{Definition}

\newcounter{theorem}
\newtheorem{lemma}[theorem]{Lemma}
\newtheorem{theorem}{Theorem}

\newenvironment{compact_enum}{
\begin{itemize}
  \setlength{\itemsep}{0pt}
  \setlength{\parskip}{2pt}
  \setlength{\parsep}{0pt}
  \setlength{\itemindent}{-5pt}
}{\end{itemize}}

\begin{document}

\title{Towards Data Quality Assessment in Online Advertising}

\numberofauthors{1} 
\author{
\alignauthor
Sahin Cem Geyik $^{1,+,*}$, Jianqiang Shen $^{2,*}$, Shahriar Shariat $^{3,+}$, \\Ali Dasdan $^{4,+}$, Santanu Kolay $^2$\\
       \affaddr{$^1$ LinkedIn Corporation, sgeyik@linkedin.com}\\
       \affaddr{$^2$ Turn Inc., \{jianqiang.shen, santanu.kolay\}@turn.com}\\       
       \affaddr{$^3$ Uber, shah.shariat@gmail.com}\\
       \affaddr{$^4$ Vida Health, ali@vida.com}
}
\date{19 May 2016}

\newcommand{\theHalgorithm}{\arabic{algorithm}}
\maketitle

\begin{abstract}
In online advertising\let\thefootnote\relax\footnote{~~~~ \\ $^*$ The authors contributed to this work equally.}\let\thefootnote\relax\footnote{~~~~ \\ $^+$ This work was completed when the authors were at Turn.}, our aim is to match the advertisers with the most relevant users to optimize the campaign performance. In the pursuit of achieving this goal, multiple data sources provided by the advertisers or third-party data providers are utilized to choose the set of users according to the advertisers' targeting criteria. In this paper, we present a framework that can be applied to assess the quality of such data sources in large scale. This framework efficiently evaluates the similarity of a specific data source categorization to that of the ground truth, especially for those cases when the ground truth is accessible only in aggregate, and the user-level information is anonymized or unavailable due to privacy reasons. We propose multiple methodologies within this framework, present some preliminary assessment results, and evaluate how the methodologies compare to each other. We also present two use cases where we can utilize the data quality assessment results: the first use case is targeting specific user categories, and the second one is forecasting the desirable audiences we can reach for an online advertising campaign with pre-set targeting criteria.

\end{abstract}




\section{Introduction}	\label{sec:intro}
Online advertising strives to serve the most beneficial advertisement (ad) to the most relevant online users in the appropriate context (a specific website, mobile application, etc.). This typically results in attaining higher return-on-investment (ROI) for the advertisers \cite{klee_2012}, where the value is generated either from a direct response such as a click or conversion (e.g. the purchase of a product, subscription to a newsletter, etc.), or through delivering a branding message. For this purpose, advertisers receive help from multiple entities in the domain. \emph{Supply-side platforms} (SSP) provide ad-space (inventory) on websites or mobile apps, to serve ad impressions to users. \emph{Ad-exchanges} run auctions on available inventory from SSPs. \emph{Demand-side platforms} (DSP) act on behalf of the advertisers and aim to bid on the most valuable inventory.
Advertisers often get performance reports from an \emph{independent evaluation agency}$~^\#$\let\thefootnote\relax\footnote{~~~~ \\ $^\#$ There are certain independent evaluation agencies in online advertising domain, whose names we cannot list here to comply with the company policy. Advertisers trust these organizations to collect the ground truth.}. For privacy reasons, these reports, in most cases, only contain aggregate metrics (e.g. click-through rate, percentage of female audiences).

In order to reach the right audience usually defined by the advertiser, which in general would improve direct response and branding metrics, the advertisers need to utilize various data sources to label the users in the most accurate way possible. \emph{Data management platforms} (DMP) have been emerging as a central hub to seamlessly collect, integrate and manage large volumes of user data \cite{elmeleegy2013overview}. Such user data could be first-party (i.e. historical user data collected by advertisers in their private customer relationship management systems), or third-party (i.e. data provided by third-party data partners, typically each specializing in a specific domain, e.g., demographics, credit scores, buying intentions). While first-party data is proprietary to the advertiser and free to utilize, third-party data often carries a pre-negotiated cost per impression (ad served to a user in a website or application). In both cases, it is important for the advertiser to know how accurate a data source is. That is, if a data source has tagged a user to be in category $c_i$ (user property, e.g. gender, age, income), how likely it is for the user to \emph{actually} be in that category.

\begin{table}[tb]
\small
\centering
\caption{Confusion matrix of data source $s$ for tagging users with category $c$.}
\begin{tabular}{cc|ccc}
\multicolumn{1}{c}{}                               &         & \multicolumn{3}{c}{Predicted by data source $s$}                          \\ \cline{3-5} 
                                                   &         & \multicolumn{1}{c|}{$c$} & \multicolumn{1}{c|}{not $c$} & Unknown \\ \hline
\multicolumn{1}{c|}{\multirow{8}{*}{\rotatebox[origin=c]{90}{ Ground Truth}}} &         & \multicolumn{1}{l|}{}      & \multicolumn{1}{l|}{}  &         \\
\multicolumn{1}{l|}{}                              & $c$       & \multicolumn{1}{c|}{$n_{+,+}$}     & \multicolumn{1}{c|}{$n_{+,-}$} & $n_{+,\bigcirc}$       \\ \cline{2-5}
\multicolumn{1}{l|}{}                              &         & \multicolumn{1}{l|}{}      & \multicolumn{1}{l|}{}  &         \\
\multicolumn{1}{l|}{}                              & not $c$   & \multicolumn{1}{c|}{$n_{-,+}$}     & \multicolumn{1}{c|}{$n_{-,-}$} & $n_{-,\bigcirc}$       \\ \cline{2-5} 
\multicolumn{1}{l|}{}                              &         & \multicolumn{1}{l|}{}      & \multicolumn{1}{l|}{}  &         \\
\multicolumn{1}{l|}{}                              & Unknown & \multicolumn{1}{c|}{$n_{\bigcirc,+}$}     & \multicolumn{1}{c|}{$n_{\bigcirc,-}$} & $n_{\bigcirc,\bigcirc}$      
\end{tabular}
\label{fig_conf_mat}
\end{table}

In this paper, we are investigating the above problem which we call \emph{data quality assessment} in online advertising. The main issue in evaluating the accuracy of a data source is the lack of ground truth in the user-level granularity. For example, the advertisers, in reality, never have access to the confusion matrix (Table~\ref{fig_conf_mat}) of a data source in either first or third-party cases. Therefore, the only way for an advertiser to evaluate the quality of a data source is to run an advertising campaign on a set of users and then evaluate the performance in hindsight. Even in those cases, the post-campaign data is often constrained (mostly due to privacy concerns) and in aggregate, that is, only the total number of users in different categories is provided, and not a granular user-to-category assignment. If it were possible to have the granular data, it would then be trivial to just use the ground truth data source to come up with the accuracy metrics, e.g. filling in the entries of the confusion matrix. Therefore, utilizing the aggregate performance statistics makes the data quality evaluation task quite challenging, and somewhat similar to aggregate learning tasks in machine learning \cite{cheplygina_2015}, few of which are also directly applicable to this problem.

The main contributions of this work are as follows:
\begin{compact_enum}
\item formal definition of data quality assessment problem, and the challenges of solving it in online advertising domain,
\item multiple approaches for evaluating the quality of a data source, which also take into account the efficiency requirements due to the large number of possible data sources$^*$\let\thefootnote\relax\footnote{~~~~ \\ * While we cannot list the exact number due to the company policy, there are currently over 200k active data sources in Turn's system.} to be evaluated,
\item several use cases where data quality assessment comes in handy for online advertising, and,
\item initial evaluation of our methodology utilizing simulated data and real-world advertising campaigns.
\end{compact_enum}

Rest of the paper is organized as follows. In Section~\ref{sec:problem_definition}, we give a more formal definition of the \emph{data quality assessment} problem. Next, we discuss the literature that deals with either data quality assessment, or aggregate learning (which, as aforementioned, is relevant to our problem) in Section~\ref{sec:previous_work}. We present our two proposed assessment methodologies in Section~\ref{sec:methodology_1} and Section~\ref{sec:methodology_2}, and later give some use cases on how we can utilize our data quality assessment output in Section~\ref{sec:use_cases}. Finally, we present some initial results in Section~\ref{sec:results} and conclude the paper along with some potential future work in Section~\ref{sec:conclusions}.

\section{Research Problems} \label{sec:problem_definition}
As we have explained in the previous section, we seek to evaluate first or third-party data sources available for online advertising using multiple accuracy metrics.
\begin{definition}
A {\normalfont sound data source} tags each virtual user {\normalfont(}cookie ID that might be specific to a browser and device{\normalfont)} with one and only one of the 3 labels -- \{Positive, Negative, Unknown\}. 
\end{definition} 
The tagging process could be explicit or deductive, but cannot be self-contradictory. For example, a user can have a positive tag -- \emph{Age25}, or a negative tag -- \emph{NotAge25}, but cannot be tagged as both \emph{Age25} and \emph{NotAge25}. The data source can also simply indicate that it has no knowledge on a user by tagging it as \emph{Unknown}. In real-time bidding, the positive tags are the most important, as advertisers usually utilize them to target the desirable audiences.
 
The problem of \emph{data quality assessment} is defined as the following:

\begin{definition} \label{def_problem}
Given a sound data source $S$, {\normalfont its data quality assessment} is defined as a measurement $m_{S}$ that has error no more than $\epsilon$ over user examples drawn from user set $U$, with probability of at least $\delta$, 
\[
  P(\Big| \mathbb{E}_{u_i\in U}\big[\Omega(u_i, S_i)\big] - m_{S} \Big| \leq \epsilon) \geq \delta, 
\]
where $\Omega(u_i, S_i)$ is a metric to measure the granular targeting performance when this data source tag user $u_i$ with $S_i$.
\end{definition}

As an example, suppose we have a data source which we utilize to tag a user as \emph{Male} (positive example) or \emph{Not Male} (negative example). Consider two evaluation metrics, which are \emph{Accuracy} (percentage of correct taggings by our data source), and \emph{True Positive Rate} (percentage of positive examples, i.e. Males, that our data source also tags as males). For the accuracy metric, we have the following $\Omega$:
\[
\Omega(u_i, S) = I(GT(u_i) == S_i),
\]
which does an exact comparison of the ground truth tagging of user $u_i$ ($GT(u_i)$) against the tagging by data source $S$ ($S_i$). On the other hand, if we were to calculate true positive rate, then we would have the following $\Omega$:
\[
\Omega(u_i, S_i) = I(GT(u_i) == S_i == \textrm{Male}) ,
\]
which counts only those cases where both ground truth and the data source tag the user as Male.

Note that the above problem definition is a very general formulation, which is typically used in evaluating Machine Learning models \cite{ferri_2009,gunawardana_2009,parker_2011}. As long as both the ground truth category of a user and that of the data source are available, one can come up with a perfect data quality assessment, i.e. $\mathbb{E}_{u_i\in U}\big[\Omega(u_i, S_i)\big] \approx m_{S}$ from Def.~\ref{def_problem}. The problem occurs when we don't have direct access to the ground truth category of every single user.  Typically, $\Omega(u_i, S_i)$ is unknown, but rather the category distribution of groups of users is provided. The main reason is to protect the privacy of users \cite{adam_1989}. In these kind of situations, especially in online advertising, we may utilize a specific data source to make smart advertising decisions to choose the most appropriate set of users, and in the end, we can receive an aggregated report from a third-party evaluator, which is considered as the ground truth and provides a non-granular distribution of the audience we have reached over many categories of interest. As an example, the report may provide that over all users we have 20\% Male, and 80\% Not Male. When this occurs, we can no longer do a one-to-one comparison between ground truth and data source in the user granularity, but rather need to come up with alternative methods that can deal with aggregated data, which is our main focus in this paper.

In many cases, we need to select the best data source from a large set of candidates with the same semantic goal and adopt it for targeting. For example, given a set of data sources that tag users as male, female, or unknown, we may care more about their \emph{relative} performance and less about their absolute measurements. The data quality assessment can then be simplified as a ranking problem:
\begin{definition}
Given two sound data sources $S^1$ and $S^2$, and an accuracy metric $\Omega$, a {\normalfont data quality ranking system} outputs a rank measurement $r^1$ for $S^1$ and $r^2$ for $S^2$ such that\\
$~~~~~ r^1 > r^2 \iff \mathbb{E}_{u_i\in U}\big[ \Omega(u_i, S^1_i) \big] > \mathbb{E}_{u_i\in U}\big[ \Omega(u_i, S^2_i) \big]$.
\label{def_rank}
\end{definition}
Once we have the rank measurements for each sound data source, we can order them and select the best one.

\section{Previous Work} \label{sec:previous_work}
As we have explained in the problem definition, evaluation of a data source can be taken as any other machine learning model evaluation task, provided that we have the ground truth information in the user granularity. A detailed evaluation of 18 performance metrics for classification problems is given in \cite{ferri_2009}. These 18 metrics can be listed as \emph{accuracy}, \emph{kappa statistic}, \emph{mean F-measure}, \emph{macro average arithmetic}, \emph{macro average geometric}, \emph{AUC of each class against the rest} (two variants), \emph{AUC of each class couples} (two variants), \emph{scored AUC}, \emph{probabilistic AUC}, \emph{macro average mean probability rate}, \emph{mean probability rate}, \emph{mean absolute error}, \emph{mean squared error}, \emph{LogLoss}, \emph{calibration loss}, and \emph{calibration by bins}. The paper provides a detailed correlation analysis and noise sensitivity analysis . Also, the survey by Gunawardana et al. \cite{gunawardana_2009} discusses both the evaluation settings and proper evaluation metrics for different classes of recommendation problems, of which online advertising is a sub-problem.

When we only have access to aggregated ground truth data, evaluation of a data source is much harder. There has been significant work in \emph{aggregate learning} tasks which utilize aggregate assignments of classes to groups of samples to train a model. Our aim in this paper is significantly different from such works, since we already have a model (i.e. data source), and we are trying to evaluate its performance utilizing many campaigns and multiple aggregates of ground truth data. Cheplygina et al. \cite{cheplygina_2015} provides an overview of aggregate learning methodologies, which may utilize granular response variables/feature vectors (single instance) or aggregate response variables/feature vectors for groups (multiple instance) to train their models, and later, testing them. Musicant et al. \cite{musicant_2007} utilizes aggregate outputs for the response variables to specialize the training process of \emph{k-nearest neighbors}, \emph{decision trees}, and \emph{support vector machines}. In \cite{chen_2006}, the authors utilize aggregate views of data, which consist of a choice of different combinations of features, response variables, and combining machine learning models learned from these views. Another interesting work is presented in \cite{yu_2015}, which gives error bounds on how a model learned from aggregate data can perform. They assert that a machine learning model should minimize \emph{empirical proportion risk}, and prove that under certain assumptions for the class distributions, learning in the aggregate setting can actually improve individual classification performance.

Finally, specific to the online advertising domain, we can list \cite{williams_2014} as being a relevant work to ours. In this paper, similar to aggregate learning techniques, the authors aim to learn a predictive model to decide whether a user is in a specific ground truth category using the aggregate data over many campaigns, by assigning the most likely label to all users in the aggregate, or assigning a probabilistic single label. They utilize logistic regression with L$_2$-norm regularization, where the response variables are the artificially generated labels.

\section{Brute Force Evaluation}  \label{sec:methodology_1}
In this section we will present our first proposal for data quality assessment, which includes setting up specialized campaigns for a data source and utilizing the targeting results directly for evaluation.

Note that we typically rely on the independent survey agencies to collect the ground truth analysis data on our audience population. Such agencies use offline data (such as credit card information) and online data (such as information filled in social networking websites) to profile an Internet user. Reports from these survey agencies are generally considered as the ground truth by advertisers. Such reports are aggregated statistics and contain no user-level information due to privacy reasons.

\subsection{Performance Campaign for Data Source} \label{subsec:brute_force}
An intuitive and straightforward way to evaluate a data source is to set up a campaign that only targets certain users which are tagged by data source $S$ to be in category $c$. This way we can calculate the quality of the data source as $p(c_g | c_s) = \frac{N(c_g)}{N(c_s)}$, where $N(c_g)$ is the number of users in category $c$ reached by this campaign as reported by the ground truth and $N(c_s)$ is the total number of users reached by the campaign via at least one impression. Note that we can put a limit on the number of impressions to be served to a user so that we can increase the unique user reach and have more reliable results.

\begin{table*}[!t]
\vspace{-8pt}
\centering
{\scriptsize
\caption{Ground truth distributions of anonymous data provider's category taggings.}
\begin{tabular}{c||c|c|c|c|c|c|c|c|c|c}
\textbf{Age Ranges} & $R_{1,g}$ & $R_{2,g}$ & $R_{3,g}$ & $R_{4,g}$ & $R_{5,g}$ & $R_{6,g}$ & $R_{7,g}$ & $R_{8,g}$ & $R_{9,g}$ & $R_{10,g}$ \\ \hline
$R_{1,p}$  & \textbf{0.414}  & 0.245 & 0.033 & 0.018 & 0.012 & 0.032 & 0.043 & 0.064 & 0.075 & 0.03  \\ \hline
$R_{2,p}$  & 0.058  & \textbf{0.297} & 0.246 & 0.037 & 0.021 & 0.043 & 0.051 & 0.089 & 0.091 & 0.046 \\ \hline
$R_{3,p}$  & 0.031  & 0.053 & \textbf{0.298} & 0.227 & 0.049 & 0.042 & 0.042 & 0.049 & 0.132 & 0.054 \\ \hline
$R_{4,p}$  & 0.031  & 0.041 & 0.089 & \textbf{0.345} & 0.182 & 0.057 & 0.037 & 0.039 & 0.116 & 0.041 \\ \hline
$R_{5,p}$  & 0.037  & 0.057 & 0.056 & 0.063 & \textbf{0.337} & 0.212 & 0.047 & 0.038 & 0.082 & 0.05  \\ \hline
$R_{6,p}$  & 0.056  & 0.049 & 0.061 & 0.041 & 0.053 & \textbf{0.339} & 0.23  & 0.041 & 0.065 & 0.046 \\ \hline
$R_{7,p}$  & 0.08   & 0.067 & 0.054 & 0.035 & 0.03  & 0.041 & \textbf{0.332} & 0.204 & 0.077 & 0.048 \\ \hline
$R_{8,p}$  & 0.065  & 0.074 & 0.082 & 0.043 & 0.03  & 0.04  & 0.048 & \textbf{0.339} & 0.203 & 0.052 \\ \hline
$R_{9,p}$  & 0.05   & 0.044 & 0.066 & 0.065 & 0.048 & 0.033 & 0.048 & 0.044 & \textbf{0.488} & 0.101 \\ \hline
$R_{10,p}$ & 0.035  & 0.036 & 0.055 & 0.048 & 0.039 & 0.048 & 0.061 & 0.06  & 0.106 & \textbf{0.492} \\ \hline
\end{tabular}
\vspace{-8pt}
\label{tab:results_for_ot}
}
\end{table*}

We applied this methodology to evaluate age and gender categorizations of some well-established data providers in online advertising. Table~\ref{tab:results_for_ot} demonstrates some results for one of the better performing such data providers in age categorization. We anonymize the name of the data provider and exact age ranges listed in the table to comply with the company's regulations. 
In Table~\ref{tab:results_for_ot}, $R_{1,\cdot}\cdots R_{10,\cdot}$ represent the age ranges such that they are mutually exclusive and sorted in ascending order, e.g. $R_5$ is the range that is the immediate higher range after $R_4$ (i.e. minimum age in range $R_5$ is one larger than maximum age in range $R_4$), and the immediate lower range before $R_6$. $R_{i,p}$ represents the predicted results by the data source for age range $i$, while $R_{i,g}$ stands for ground truth data for the same age range.
For example, we can observe from the table that $p(R_{4,g} | R_{4,p}) = 0.345$, that is, when our data source classifies  the user to be in age range $R_4$, 34.5\% of the time it is correct, or in other words, 34.5\% of the users reached by campaign that targets category $R_4$, as predicted by the data source, were actually in category $R_4$, as provided by ground truth. Note that for this particular data source, the exact match, highlighted in bold, is quite high compared to random.

Data providers utilize various models and online/offline information to tag users.  Please note that data sources from the same data provider may have quite diverse accuracy values. For example, the accuracy results in Table~\ref{tab:results_for_ot} range from 0.29 to 0.49. Thus we cannot evaluate one data source and assume its sibling data sources have similar predictive power. Each data source needs to be evaluated individually. This causes a significant disadvantage with the above methodology, which is the fact that we need to set up a separate campaign for each category so that we can gather the accuracy statistics. To remedy this problem, we propose a second methodology in Section~\ref{sec:methodology_2}, which solves an optimization problem to come up with the best fitting accuracy probabilities based on the aggregate reports.

\subsection{Cost Analysis} \label{subsec_cost_analysis}

In this subsection, we further analyze the cost of the brute force method discussed in Section~\ref{subsec:brute_force}. It must be noted that obtaining the ground truth, that is the aggregated labeled data, is costly on its own right. However, in the following lemma, we focus only on the ad serving costs to underline the utility and benefit of our approach.
\begin{lemma}
\label{lemma_cost}
Given a data source that can tag the user by one of the possible $c$ categories (e.g. the data source gives a positive/negative output on one age group of possible $c$ age groups), then to observe a significant difference between the calculated accuracy of one category versus others, we need\\ at least $\lceil \frac{4202.969}{c}(1-\frac{1}{c}) \rceil$ impressions.
\end{lemma}
\begin{proof}
Assuming a uniform distribution, we assume the average on-target rate is $\frac{1}{c}$ for each data source (although the intention of a data source is to increase this value). Each impression can be considered as a Bernoulli trial with $\frac{1}{c}$ probability of success. The sample variance is, thus, $\frac{1}{c}(1-\frac{1}{c})$. We would like to detect a significant difference between the prediction accuracy of the correct category versus the rest of possible tags. The industry standard accepts a $5\%$ error. Then, for a significance level of $\alpha=5\%$ for a two tailed hypothesis test and to attain at least $90\%$ power, we have $\frac{0.05\sqrt{n}}{\sqrt{\frac{1}{c}(1-\frac{1}{c})}}>\left(z_{0.975}+z_{0.9}\right)$, where $n$ stands for the number of users, $z_{0.975}$ and $z_{0.9}$ are the values of the quantile function of the standard normal distribution for $97.5\%$ and $90\%$, respectively \cite{Casella02}. Therefore, the number of users that receive the ad impressions must be more than $\lceil \frac{4202.969}{c}(1-\frac{1}{c}) \rceil$ for each data source.    
\end{proof}
Based on Lemma~\ref{lemma_cost}, for $d$ data sources that provide information on one of the possible $c$ categories, we need at least $d \lceil \frac{4202.969}{c}(1-\frac{1}{c}) \rceil$ impressions. As we discussed before, it is necessary to evaluate each data source individually, considering the diversity of their predictive power. 
Given the very large number of data sources, this causes the brute-force approach to incur a very significant cost.

\section{Accuracy Inference} \label{sec:methodology_2}
High-quality data sources can enable advertisers to reach the right audience at the right moment. Because they have become an important component of online advertising, more and more online/offline data are being ingested into Turn's data management platform. As we have mentioned previously, there are currently over 200k active data sources in our system. Lemma~\ref{lemma_cost} established that explicitly evaluating each of these data sources by running performance campaigns is overwhelmingly costly: not only a large amount of money is required to run the performance campaigns, but also enormous manual effort to set up and manage those campaigns is essential as well. We need an efficient way to simultaneously infer the accuracy of multiple data sources.

As we have presented in Section~\ref{sec:problem_definition}, our focus in this paper is to calculate the accuracy metrics of a data source for single or multiple categories. In essence, we are trying to calculate a set of probabilities, which represent the likelihood of a data source predicting correctly/incorrectly that a user belongs to a category $c_i$. In Figure \ref{fig_basic_objective}, we have shown the set of probabilities that we aim to predict. For representational purposes, we have shown the accuracy probabilities of a data source which denote its capabilities to tag a user as \emph{Male} or not, though the same logic follows for any category. The probabilities in the figure can be listed as follows:
\begin{compact_enum}
\item $\alpha_1$: The probability of the user actually being Male when the data source tags it as Male. This value can also be called \emph{precision} or \emph{positive predictive value}.
\item $\alpha_2$: The probability of the user actually being Not Male when the data source tags it as Male. This value can also be called \emph{false discovery rate}.
\item $\alpha_3$: The probability of the user being Unknown (i.e. ground truth does not exist) when the data source tags it as Male.
\item $\beta_1$: The probability of the user actually being Male when the data source tags it as Not Male.
\item $\beta_2$: The probability of the user actually being Not Male when the data source tags it as Not Male. This value can also be called \emph{negative predictive value}.
\item $\beta_3$: The probability of the user being Unknown (i.e. ground truth does not exist) when the data source tags it as Not Male.
\item $\gamma_1$: The probability of the user actually being Male when the data source tags it as Unknown.
\item $\gamma_2$: The probability of the user actually being Not Male when the data source tags it as Unknown.
\item $\gamma_3$: The probability of the user being Unknown (i.e. ground truth does not exist) when the data source tags it as Unknown.
\end{compact_enum}
As it can be seen from the figure, and trivial from the definitions, we have $\alpha_1 + \alpha_2 + \alpha_3 = \beta_1 + \beta_2 + \beta_3 = \gamma_1 + \gamma_2 + \gamma_3 = 1$. Among these nine variables, $\alpha_1$, i.e. \emph{precision}, is often the most important value for the advertisers, since it denotes the goodness of a data source to be used for their advertising purposes. To calculate this value, we presented the methodology, which is based on creating specific campaigns to evaluate a singular data source for a specific category in Section~\ref{sec:methodology_1}. In this section we are proposing an optimization scheme, which utilizes the aggregated category distributions over multiple campaigns, for both the data source we want to evaluate, and the ground truth. In the following subsections, we call the above nine variables \emph{predictive values} of a data source.

\begin{figure}[!t]
\centering
\includegraphics[width=2.2in]{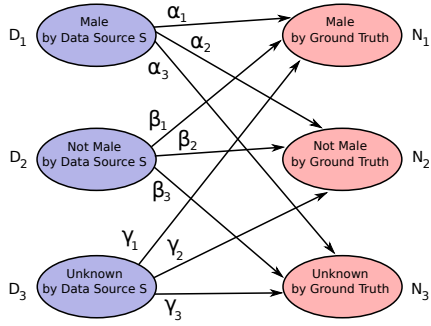}
\caption{Objective probabilities of a data source $S$ for a specific category (\emph{is user Male}).}
\label{fig_basic_objective}
\vspace{-6pt}
\end{figure}

\subsection{Setup for Inference} \label{subsec:methodology_setup_camp}
We propose to set up multiple performance campaigns without using any data source for targeting, so the audience will not be explicitly skewed by any data source. We compare the ground truth of each campaign against the hypothesis of each data source, and infer the quality of the data source.

We follow a set of rules to set up a performance campaign. First, the targeting criteria should be minimal and cannot be biased by any third-party data. For example, targeting the online users in U.S. is fine, since this is purely based on IP address and not biased; however, using a data source to limit audience to \emph{middle-aged men} is not acceptable as the quality of this data source is what we want to assess. In general, only geographical location should be used as targeting criteria. Second, the targeted websites must be discriminative, that is, the population of the visiting users should be largely skewed towards one of our possible tags. This way, we will not mistakenly estimate a data source to be accurate, when in fact it is predicting the label in a random manner. For example, a website is beneficial for such an experiment if 70\% of visitors are female and 30\% of visitors are male, and not necessarily if the distribution is 50\% to 50\% (since, in such a case, a random prediction of female vs. male will closely fit the overall audience in aggregate). One should note that obtaining such knowledge is not always feasible before running the campaign. Therefore, in our system, we target only the websites that are known, based on our domain experience and verified by independent reports,
 to be popular among a certain group of audience. After creating such a campaign, we run it for a certain period to collect data. The ground truth is collected through independent agencies as described in Section~\ref{sec:methodology_1}.

We log the received report data along with the first-party campaign data into our in-house data warehousing system called Cheetah \cite{chen2010cheetah}, which is built on top of the Hadoop framework \cite{hadoop_2012}. Cheetah is designed specifically for our online advertising application to allow various simplifications and custom optimizations. Campaign facts are stored within nested relational data tables. Fast MapReduce jobs are designed to extract key features of the performance campaigns, compare them with the ground truth and infer the accuracy of a data source. Utilizing the collected campaign information, we present two approaches in this section to efficiently infer the quality of data sources: one that ranks data sources, and another which directly deduces precision ($\alpha_1$) of a data source.

\subsection{Ranking Based Assessment} 
\label{subsec:methodology_rank}

{\bf Ranking Data Sources.}
In many instances, we only need to choose the best data source from a large set of candidates with similar semantic purposes. If this is the case, then data quality assessment becomes a ranking problem. A data source's absolute precision $\alpha_1$ is of less importance then, and rather its rank among others is critical.  

Since the independent evaluation agency sends us the aggregate statistics on a campaign, we can similarly construct such statistics using a data source. Note that this approach will only represent the view of the data source, and not the ground truth, unlike the independent agency case. We can then evaluate this data source based on how close the constructed statistics are to that of the ground truth, and therefore rank data sources based on such closeness measure. This logic is presented in Algorithm~\ref{alg_rank}.

\begin{algorithm}[]
 \KwIn{$\psi$: aggregation function, $f$: closeness function} 
 \KwIn{$S$: data source}
 \KwOut{$\bar{\mu}$: score for quality of data source $S$} 
  \ForEach{ \normalfont{performance campaign} $\mathcal{C}$ }{ 
     $\mathcal{U}$ $\leftarrow$ Retrieve audience of $\mathcal{C}$\;      
  	 $\hat{R}$ $\leftarrow$ $\psi(\mathcal{U}, S)$\; 
  	 ${R}$ $\leftarrow$ Retrieve the ground truth report\;
  	 $\mu_{\mathcal{C}}$  $\leftarrow$ $f(\hat{R}, R)$\;
  }
  $\bar{\mu}$ $\leftarrow$ Calculate the average value of $\mu_{\mathcal{C}}$\;
  \Return $\bar{\mu}$\;
 \caption{Measurement calculation for ranking.}
 \label{alg_rank}
\end{algorithm}

There are multiple ways to design the closeness function $f$ to compare two aggregated statistics. Since the positive taggings are the most valuable for online advertising purposes, we propose to compare the positive population distributions between the ground truth and the data source. We define the percentage of population marked as positive by a data source as 
\begin{align}
\hat{R} = \frac{count_{+}(S, \mathcal{U}) } { count_{+}(S, \mathcal{U}) +count_{-}(S, \mathcal{U}) },
\end{align}
where $count_{+}$ counts the number of users in $\mathcal{U}$ marked as positive by $S$, and $count_{-}$ counts negatives. Given that $R$ is the ground truth ratio of positive population, a simple way to calculate the closeness can be defined as:
\begin{align}
\big| \hat{R} - R \big|.
\end{align}
However, this does not consider the scale of $\hat{R}$ or $R$, which are usually quite small for a rare positive group. To make the measurements more comparable, instead we propose to calculate the relative error as the closeness:
\[
RelativeErr = \frac{\big| count_{+}(\mathcal{U}) - \hat{count}_{+}(S, \mathcal{U}) \big|} { count_{+}(\mathcal{U}) },
\]
where $count_{+}(\mathcal{U})$ is the number of positives reported by the ground truth, 
$count_{-}(\mathcal{U})$ is the number of ground truth negatives, and $\hat{count}_{+}(S, \mathcal{U})$ is the scaled number of positives marked by the given data source.
We want to scale the number of positives of the data source, because the population recognized by $S$ might be quite different from the one recognized by the ground truth. For example, the independent evaluation agency might have data on 10000 users, while $S$ might only have data on 1000 users. Therefore, we need to extrapolate the population unrecognized by $S$ to scale up the populations:
\begin{align}
RelativeErr & = \frac{\big| count_{+}(\mathcal{U}) - \hat{R}\cdot ( count_{+}(\mathcal{U}) + count_{-}(\mathcal{U})) \big|} { count_{+}(\mathcal{U}) } \nonumber \\
 & = \big|1 - \frac{ \hat{R} } { count_{+}(\mathcal{U}) / ( count_{+}(\mathcal{U}) + count_{-}(\mathcal{U})) } \big| \nonumber \\
 & = \big|1 - \frac{ \hat{R} } { R } \big|
  =  \frac{ \big| R - \hat{R} \big| }  { R }
\label{relative_err_form}
\end{align}

~

The above value reflects the potential error rate if we scale the data source's recognizable population to the size of the ground truth population. Per Algorithm~\ref{alg_rank}, we calculate average relative error ($RelativeErr$) across all performance campaigns for each data source. We can then rank data sources based on their average relative errors.

{\bf Soundness Analysis.}
A ranking algorithm needs to be sound to ensure the optimal assessment: Given two data sources $S^1$ and $S^2$, a ranking algorithm is \emph{sound} if it outputs measurements $r^1$ for $S^1$ and $r^2$ for $S^2$, such that $r^1 < r^2$ if and only if $S^1$ is more likely to perform better than $S^2$, as we defined in Def.~\ref{def_rank}. We will show that the $RelativeErr$ based ranking algorithm is sound in many cases. First, let us define the notion of \emph{unbiasedness} for a data source:
\begin{definition}
A data source $S$ is {\normalfont unbiased} if and only if its positive predictive value equals to its negative predictive value: $\alpha_1 \approx \beta_2$. 
\end{definition}
Our experience suggests that many data sources we utilize are on demographics and can be considered as \emph{unbiased}. For example, the accuracy that a data source claims someone as male, in general, is close to the accuracy that it claims someone as female (\emph{Not Male} as in Figure~\ref{fig_basic_objective}). In real-time bidding, better \emph{on-target metrics}, i.e. improving the ratio of audience that really have the data source's claimed characteristics, is the endeavor of any data provider. We can show that the $RelativeErr$ based ranking algorithm is sound:
\begin{lemma}
\label{lemma_rank_correct}
Given a set of sound data sources $\langle S^1, .., S^k \rangle$, the $RelativeErr$ based ranking algorithm is sound for precisions and orders the data sources based on their expected performance. In addition, if the data sources are all unbiased, the algorithm is sound for any on-target metrics.
\end{lemma}
\begin{proof}
Per definition, $\hat{R}$ of the better data source is closer to the reality, thus $\big| R- \hat{R} \big| $ is smaller. Since $R$ is constant, the order of $RelativeErr$ is preserved for precisions. 
On-target metrics can be the precision, or the negative predictive value, or simply the micro or macro average of the two predictive values. Since the data sources are unbiased, the order of metrics for negatives are also preserved. Averaging is monotonic, therefore we can expand the previous statements to micro and macro averaging cases as well. 
\end{proof}

\subsection{Precision Inference Approach} 
\label{subsec:methodology_infer}

Although the ranking methodology is able to pinpoint the highest performing data sources, the output ranking measurement is only a surrogate of precision. It correlates with the underlying precision, but is inherently different. As we will show later, in online advertising, it is often necessary to forecast the campaign performance as well as evaluate whether a third-party data source is worth the extra amount of money that an advertiser has to pay in order to utilize it. In such cases we need an accurate estimation of a data source's precision.

{\bf Direct Inference of Data Source Precision.}
We propose an efficient way to directly estimate the predictive values of a data source. As shown in Figure~\ref{fig_basic_objective}, a data source's hypothesis on the audience population can be mapped into the ground truth using its predictive values. Given a performance campaign $\mathcal{C}^i$, let the size of \emph{Positive}, \emph{Negative} and \emph{Unknown} audiences identified by data source $S$ be $D^i_+$, $D^i_-$ and $D^i_\bigcirc$ correspondingly (these are scaled values to cover the whole population of $\mathcal{C}^i$), and the size of ground truth \emph{Positive}, \emph{Negative} and \emph{Unknown} audiences be $G^i_+$, $G^i_-$ and $G^i_\bigcirc$ correspondingly. When the audience population size is large, it is clear that we have
\begin{align*}
G^i_+ \approx D^i_+\cdot \alpha_1 + D^i_- \cdot \beta_1 + D^i_\bigcirc \cdot \gamma_1 \\
G^i_- \approx D^i_+\cdot \alpha_2 + D^i_- \cdot \beta_2 + D^i_\bigcirc \cdot \gamma_2 \\
G^i_\bigcirc \approx D^i_+\cdot \alpha_3 + D^i_- \cdot \beta_3 + D^i_\bigcirc \cdot \gamma_3 
\end{align*}
Combining with the probability simplex constraint and the unbiasedness constraint, we can estimate a data source's predictive values by solving a quadratic optimization problem: given $n$ performance campaigns $\langle \mathcal{C}^1, \mathcal{C}^2, .., \mathcal{C}^n \rangle$, we search for predictive values $\alpha_1, \alpha_2, \alpha_3, \beta_1, \beta_2, \beta_3, \gamma_1, \gamma_2, \gamma_3$ so that 
\begin{align}
\min_{\alpha, \beta, \gamma} & \sum_{i=1}^{n} \Big( (D^i_+\cdot \alpha_1 + D^i_- \cdot \beta_1 + D^i_\bigcirc \cdot \gamma_1 - G^i_+)^2  \nonumber \\
+ & (D^i_+\cdot \alpha_2 + D^i_- \cdot \beta_2 + D^i_\bigcirc \cdot \gamma_2 - G^i_-)^2 \nonumber \\
+ & (D^i_+\cdot \alpha_3 + D^i_- \cdot \beta_3 + D^i_\bigcirc \cdot \gamma_3 - G^i_\bigcirc)^2 \Big) \label{rank_obj} \\
s.t.~~
& 0\leq \alpha_j, \beta_j, \gamma_j \leq 1 ~~~ \forall j=1, 2, 3 \label{rank_noneg} \\
& \sum_{j=1}^{3}\alpha_j=1,~~ \sum_{j=1}^{3}\beta_j=1,~~ \sum_{j=1}^{3}\gamma_j=1 \label{rank_simplex} \\
& -\xi \leq \alpha_1 - \beta_2 \leq \xi \label{rank_unbiased}~.
\end{align} 
Here, (\ref{rank_obj}) is our optimization objective which aims to find the best mapping between the data source's hypothesis and the ground truth. Note that here we assume the size of audiences of campaigns $\mathcal{C}^i$ to be similar, which can be controlled at campaign set up time. Otherwise we need to normalize by the audience size of each campaign. (\ref{rank_noneg}) and (\ref{rank_simplex}) enforce the probability simplex. (\ref{rank_unbiased}) attempts to help us find the unbiased solution, and predefined constant $\xi$ controls our confidence on the unbiasedness of the predictive values.

We, therefore, can run a few performance campaigns, extract each data source's hypothesis on those campaigns, compare with the ground truth and solve the above optimization problem. As we will show, this will efficiently give us the estimated predictive values of data sources in batch (among those, \emph{precision} is the most valuable for online advertising).

{\bf Performance Analysis.}
The proposed inference approach is efficient, in terms of both computation complexity and money. First, it is straightforward to show that the quadratic programming problem has a semi-definite Hessian with a bowl shape. The optimization problem is convex and can be solved efficiently with polynomial time complexity. Additionally, we only need to run a limited number of performance campaigns to simultaneously estimate the predictive values of multiple data sources. In practice, it is possible that a data source's predictive values are slightly different in different performance campaigns due to variance. Given a campaign $\mathcal{C}^i$, it is natural to assume a data source's predictive values for this specific campaign are
\begin{align*}
\alpha_j^i = \alpha_j + \varepsilon_j^1, ~~\forall j= 1, 2, 3,  \\
\beta_j^i = \beta_j + \varepsilon_j^2, ~~\forall j= 1, 2, 3,  \\
\gamma_j^i = \gamma_j + \varepsilon_j^3, ~~\forall j= 1, 2, 3, 
\end{align*}
where $\varepsilon$ is normally distributed with zero mean. In such cases, we can get the unbiased estimate of a data source's predictive values by running a limited number of performance campaigns:
\begin{theorem}
Given $k \geq 3$ performance campaigns, our direct inference method can get the unique and unbiased estimate of a data source's predictive values. Furthermore, given any predictive value $\alpha_i$ and its estimation $\hat{\alpha_i}$, we have 
\[
P(\Big| \alpha_i -  \hat{\alpha_i} \Big| \leq s_{i} \cdot T_{\delta/2, 2k-6}) \geq 1-\delta,
\]
where $0 < \delta < 1$ is a constant, $s_{i}$ is the standard error of the estimation, and $T_{\delta/2, 2k-6}$ is $(1-\delta/2)$-th quantile of Student Distribution with $2k-6$ degrees of freedom. 
\end{theorem}
\begin{proof}
The optimization can be converted into a linear regression problem within a simplex search space. This regression problem contains 6 free regressors and each campaign provides 2 points in the space. When we have  $k\geq 3$ campaigns, the quadratic matrix is positive-definite and we will have a unique global optimal solution. A Bias-Variance-Noise decomposition shows the solution is unbiased. 

Since the errors are normally distributed, the sum of the regression residuals is then distributed proportional to Student Distribution with $2k-6$ degrees of freedom:
\[
t=(\alpha_i -  \hat{\alpha_i})/s_{i} \sim T_{2k-6}
\]
We then construct the confidence levels for the estimated regressors.
\end{proof}
By running more campaigns, we can quickly reduce the estimation errors and get highly reliable predictive value estimations of multiple data sources. Given its computational and economic efficiency, we adopted the direct inference method and utilize it continuously to generate the quality report on data sources.

\section{Use Cases} \label{sec:use_cases}
In this section, we will discuss some use cases where the quality assessment of first or third-party data sources can be useful. First we will talk about targeting in online advertising, and the amount that an advertiser should be willing to pay for a data source. Then we will give a very general use case in campaign forecasting, i.e. to predict, before an online advertising campaign starts, what category of users will actually be reached by a pre-set targeting criteria.

\subsection{Targeting in Online Advertising} \label{subsec:targeting_use_case}
Advertisers aim to reach the best audiences to promote their products, so that they can increase the likelihood of a click or an action happening. The automated way of grouping users into beneficial and non-beneficial subsets is often called audience segmentation. For an informative work on how this kind of audience segmentation can improve click rates, refer to \cite{yan_2009}. In this paper, however, we focus on a different kind of targeting where the advertisers already have a pre-defined set of users they want to target. As an example, suppose that an advertiser wants to reach only female audiences within the age range 21-35. There are multiple data sources this advertiser can utilize to reach this group, but as discussed in this paper, none of these data sources gives a definitive classification. Intuitively, an accurate prediction of the quality of a data source is essential for advertisers to choose it over others. Also, note that here we mostly care about the \emph{precision} or \emph{positive predictive value} of a data source (i.e. $\alpha_1$, when a data source suggests that a user is in category $c$, the likelihood that this user actually belongs to category $c$), since this is the signal that the advertiser uses to bid on a user.

Here, we would also like to discuss the consideration of \emph{data cost}. In general, when an advertiser wants to utilize a third-party data source for bidding purposes, it should pay the third-party provider a certain amount of money. This cost is generally per impression served using this data source, hence can have significant effect on the ROI of an advertiser (i.e. advertiser needs to pick up extra clicks/conversions to make up for the money paid to the third-party for the targeting information it provides). An important point for an advertiser to consider is if using a data contract is ``worth'' its price. We will give a simple calculation here for the case when the advertiser utilizes no data sources to reach a specific audience (i.e. free targeting), and whether adding the data source and paying for it makes sense. Our main argument is that, by paying for the data source to target a specific audience, the reduced cost of the mis-targeted impressions (i.e. those impressions that are served to the audience that are out of our desired audience) should make up for the data cost. In other words, for the same amount of money we should get more of the desired impressions, although our total number of impressions is less due to data cost. Please note that below we assume the effective cost per impression (cpi) to be the same for both free targeting and data source assisted targeting, just that the data source has the additional data cost per impression (cpi$_{data}$):
\[
\frac{\textrm{totalSpend}}{\textrm{cpi}+\textrm{cpi}_{\textrm{data}}} (1-\textrm{errorRate(dataSource)}) \geq ~~~~~~~~~~~~
\]
\vspace{-12pt}
\[
~~~~~~~~~~~~~~~~~~~~~~~~~~~~~~  \frac{\textrm{totalSpend}}{\textrm{cpi}} (1-\textrm{errorRate(freeTargeting)})
\]
\[
\frac{\alpha_{1,\textrm{data}}}{\textrm{cpi}+\textrm{cpi}_{\textrm{data}}} \geq \frac{\alpha_{1,\textrm{freeTargeting}}}{\textrm{cpi}} ~~ .  ~~~~~~~~~~~~~~~~~~~~~~~~~
\]
Above, \emph{totalSpend} is the amount of money that the campaign spends, \emph{cpi} is effective cost per impression, and \emph{cpi}$_{data}$ is data cost per impression (hence $\frac{\textrm{totalSpend}}{\textrm{cpi}+\textrm{cpi}_{\textrm{data}}}$ is the number of impressions picked up by data source assisted targeting, and $\frac{\textrm{totalSpend}}{\textrm{cpi}}$ is the number of impressions that can be picked via free targeting, i.e. no data cost). errorRate is indeed the inaccuracy of free targeting (percentage of audience that is not desired), and percentage of the cases when the data source predicts a user to be in desired audience, while, in fact, it is not. In the next inequality, we actually translated (1 - errorRate) into $\alpha_1$ from Section~\ref{sec:methodology_2}. After further reorganizing the above inequality, we get the following:
\begin{equation}
\label{eq_max_data_cost}
\textrm{cpi}_{\textrm{data}} \leq \textrm{cpi} \times \left( \frac{\alpha_{1,\textrm{data}}}{\alpha_{1,\textrm{freeTargeting}}} - 1 \right)  ~~ .
\end{equation}
This means that for a data source to be beneficial for a campaign, its data cost per impression should be less than $\textrm{cpi} \times \left( \frac{\alpha_{1,\textrm{data}}}{\alpha_{1,\textrm{freeTargeting}}} - 1 \right)$. Please note that we have the assumption here that effective cost per impression would be the same for free targeting vs. data source which is not always valid, i.e. we may have to pay more to show ads (impression) to those users that the data source tagged to be desirable. Also, it can be seen that the benefit of the data source often depends on how expensive the impressions are for a campaign, hence is campaign specific. Finally, the above calculations do not take into account the cost of data evaluation utilizing our proposed two methodologies, which was also mentioned in Section~\ref{subsec_cost_analysis}. However, this evaluation can be performed once for each data source, and hence is not of significance for each campaign that utilizes it.

\subsection{Forecasting} \label{subsec:forecasting_use_case}
Forecasting the performance (return-on-investment), reach (unique users we can show an ad to), and delivery (amount of money we can spend on advertising given the targeting criteria) of a campaign is a significant problem that has to be dealt in online advertising \cite{jalali_2013}. Here we show that by utilizing the accuracy metrics (i.e. $\alpha_{1 \rightarrow 3}$, $\beta_{1 \rightarrow 3}$, and $\gamma_{1 \rightarrow 3}$ from Section~\ref{sec:methodology_2}) over multiple data sources that may tag a user, we can actually predict the expected number of users that will fall into a specific audience/category, on top of the total spend/reach as in the traditional forecasting problem, once the advertising campaign goes live.

Here is how the forecasting process in online advertising works. Once the advertiser sets some targeting criteria (filtering of users to show ads according to anonymous user properties) and goals (in terms of clicks and conversions) for a campaign, we can utilize our system as explained in \cite{jalali_2013} to find out which users this campaign is likely to reach. We can already calculate expected number of unique users and delivery for the campaign from this information alone. Furthermore, one problem that we can solve for the advertisers is the prediction of what percentage of these users will fall into a specific user category/class $c$. Note that the approach mentioned in \cite{williams_2014} does work on this problem of predicting likelihood of a user belonging to a specific category via utilizing many features, but their focus is on targeting rather than forecasting. Here, we suggest that rather than training a simple model for predicting the membership in a category, we can utilize multiple data sources and their estimated predictive values to forecast the \emph{expected} number of users that will fall into a category. This information is not used in bidding time, hence contrary to the targeting use case we explained in Section~\ref{subsec:targeting_use_case}, there is no data cost.

\begin{figure}[htb]
\centering
\includegraphics[width=2.8in]{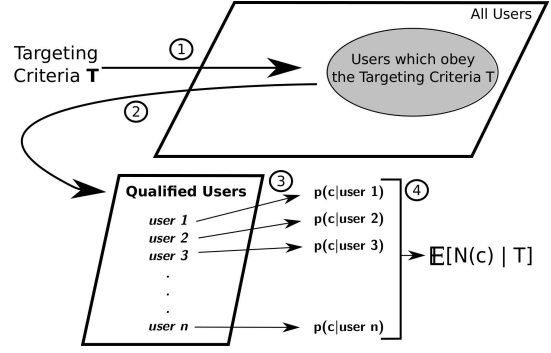}
\caption{Forecasting Example.}
\label{fig_forecasting_example}
\end{figure}

Figure \ref{fig_forecasting_example} summarizes the overall idea. In the first step, we communicate the targeting criteria set by the advertiser to our set of users. For real-time forecasting, we often need to use a sampled set of users \cite{jalali_2013}. Once we filter the users that are appropriate for the targeting criteria, we can go through each of these users and see whether these users are tagged by a first or third-party category $c_d$. Once we see these taggings, we can calculate the probability of this user belonging to a desired ground-truth category $c_g$ by the probability $p(c_g|c_d)$ which is the $\alpha_1$ from Section~\ref{sec:methodology_2}, i.e. precision. If we assume that each user is tagged by one and only one $c_i$, we can \emph{forecast} the expected number of users that will belong to category $c_g$ as:
\begin{equation}
\mathbb{E}[~|c_g|~] = \sum_{u \in T} p(c_g | u) = \sum_{u \in T} \sum_{c_i \in C} I(u \textrm{ has } c_i) ~ p(c_g | c_i) ~ .
\label{eq_exp_in_cat}
\end{equation}
In (\ref{eq_exp_in_cat}), $T$ is the set of users that belong to a certain set of targeting criteria T. $c_g$ is the category that the advertiser desires to forecast how many of their targeted users will belong to. $c_i$ is a first/third-party category from a list of categories $C$ for which we have the prediction values. $I$ is an indicator to see whether a user has category $c_i$, and above formula is valid only because we assume that each user has only one of possible $c_i$s. If each user can have multiple first/third-party categories (as is the case in real situations), we need to aggregate multiple $p(c_g|c_i)$s, where we can utilize combination methods such as getting the maximum, minimum, average or median of the precision values.

\section{Experimental Results} \label{sec:results}
In this section we will give some preliminary results for our optimization-based evaluation technique. We have already given some preliminary results for our first methodology (Section~\ref{subsec:methodology_setup_camp}) in Table \ref{tab:results_for_ot}. As aforementioned, we do aim to calculate all nine prediction values of a data source for a category, but for purposes of online advertising, the most important one is the precision ($\alpha_1$). Only if this value is high we can reliably use this data source to reach a certain category of users.

{\bf Simulation Results}.
To evaluate our methodology we ran several simulated campaigns, where for each campaign we create an audience of 100 users and assign them to \emph{predicted} categories as follows:
\begin{itemize}
\item Random, disjoint sets of 20 users each are assigned to category \emph{c}, \emph{not c}, and \emph{unknown},
\item The rest 40 users are assigned to either category \emph{c}, \emph{not c}, and \emph{unknown} in a uniform manner.
\end{itemize}
Then, we generate the ground truth categories for two types of data sources with the following \emph{actual} probability values:
\begin{itemize}
\item \emph{High Quality:} This data set has the following underlying nine probability values: $\alpha_{1\rightarrow3}$ = (0.8, 0.15, 0.05), $\beta_{1\rightarrow3}$ = (0.2, 0.7, 0.1), $\gamma_{1\rightarrow3}$ = (0.4, 0.5, 0.1), note that $\sum_{1\rightarrow3}\alpha_i = \sum_{1\rightarrow3}\beta_i = \sum_{1\rightarrow3}\gamma_i = 1$,
\item \emph{Low Quality:} This data set has the following underlying nine probability values: $\alpha_{1\rightarrow3}$ = (0.4, 0.5, 0.1), $\beta_{1\rightarrow3}$ = (0.3, 0.6, 0.1), and $\gamma_{1\rightarrow3}$ = (0.5, 0.4, 0.1).
\end{itemize}
Please note that this kind of synthetic data generation is quite counter-intuitive. We first create the predicted values using some pre-set distribution, and then generate these users' actual categories using the predictive values of the two data sources. For example, if the user in our synthetically generated audience has a predicted category of \emph{not c} by \emph{High Quality} data source, then we assign it to ground truth category of \emph{c} by probability $\beta_1=0.2$, \emph{not c} by probability $\beta_2=0.7$, and \emph{unknown} by probability $\beta_3=0.1$.

Once we generate the dataset, we actually have aggregated values of category counts for each data source. Using these category counts, we can utilize our data quality assessment method we described in Section~\ref{subsec:methodology_infer}, and look into the difference between our computed predictive values, and the actual predictive values as given above.

\begin{figure}[!t]
\centering
\includegraphics[width=2.1in]{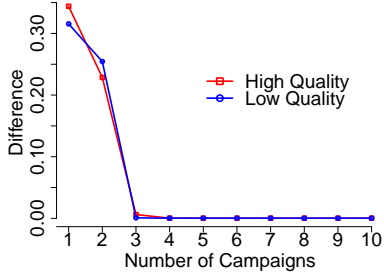}
\caption{Number of campaigns versus difference between predicted and actual precision values.}
\label{fig_diff_over_num_camp}
\end{figure}

\begin{figure}[!t]
\centering
\includegraphics[width=2.1in]{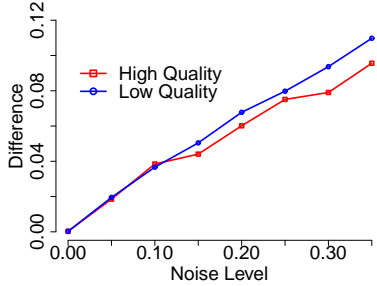}
\caption{Noise versus difference between predicted and actual precision values for six campaigns scenario.}
\label{fig_diff_vs_noise}
\vspace{-6pt}
\end{figure}

The results of the above described simulations are given in Figures \ref{fig_diff_over_num_camp} and \ref{fig_diff_vs_noise}. In Figure \ref{fig_diff_over_num_camp}, we estimate the nine predictive values for each of the data sources using our methodology, by utilizing multiple campaigns (the results are averaged over 100 trials at each value in x-axis). As we have proven in Section~\ref{subsec:methodology_infer}, we do need at least three campaigns to get a unique solution for all nine probabilities. We can observe that starting with four campaigns, the difference of the real values versus our predicted values fall to zero. We have plotted the difference between $\alpha_1$s (precision values, since $\alpha_1 = p(\textrm{actually positive}~|~\textrm{predicted positive})$), but the results are similar for other $\alpha_{2\rightarrow3}$, $\beta_{1\rightarrow3}$, and $\gamma_{1\rightarrow3}$.

Next, we performed another experiment where we introduced a uniform noise between $\pm \zeta$ (where we changed $\zeta$ between 0 and 0.35) to the above nine real predictive values, and then generated the ground truth assignments. We tried to recover the nine predictive values using six campaigns and present the difference (averaged over 100 trials) between real and predicted $\alpha_1$ values in Figure \ref{fig_diff_vs_noise}. We can see that even under significant noise levels, our methodology can recover the precision values accurately. Because $\alpha_1$ for the high quality data source is higher, we can also observe that the noise effect is slightly less.

{\bf Real-World Results}.
Following the methods we discussed in Section~\ref{subsec:methodology_setup_camp}, we ran 156 performance campaigns, each of which targeted a specific website. We used half of the campaigns to calculate $RelativeErr$ and predictive values for around 100 data sources. Then, we tried to estimate the positive population in the rest of campaigns using these 100 data sources. We utilized the average of positives predicted by the sources, and calculated the correlation with the ground truth positive sizes. For the direct inference method, it is clear that for each campaign $\mathcal{C}^i$, its estimated positive population is $G^i_+ \approx D^i_+\cdot \alpha_1 + D^i_- \cdot \beta_1 + D^i_\bigcirc \cdot \gamma_1$ (for a single data source). For the $RelativeErr$ method, by deducing (\ref{relative_err_form}), we can roughly estimate $G^i_+= M \cdot \hat{\tau}  \hat{R} / (1 - \bar{\mu})$, where $M$ is the population size, $\hat{\tau}$ is the percentage of population recognized by the ground truth in the training set, and $\bar{\mu}$ is the average $R-\hat{R}/R$ (i.e. \emph{RelativeErr} from Eq.~\ref{relative_err_form}) of the training set for a single data source. 

Each method's Pearson correlation coefficient is shown in Figure \ref{fig_correlate}. The direct inference method gives a significantly more accurate estimate of the positive population ($p<.001$) and it correlates well with the ground truth.

\begin{figure}[tb]
\centering
\includegraphics[width=2.5in]{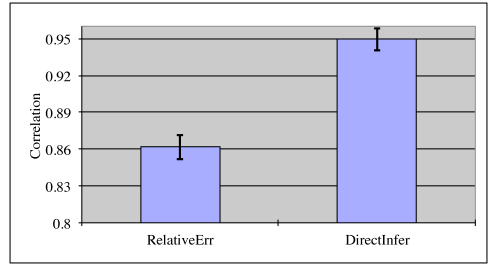}
\caption{Pearson correlations between the  estimated positive population and the ground truth.}
\label{fig_correlate}
\end{figure}

Next, we utilized two popular data sources $S1$ and $S2$ as the targeting criterion and ran test campaigns individually. The reported positive rates from the independent evaluation agency can be treated as the ground truth of their precisions. The ground-truth precisions and our estimated values are listed in Figure~\ref{fig_real_camp}. The direct inference method yields a much closer estimation to the ground truth ($p<.001$), while the ranking method preserves the orders but the values are substantially different from the ground truth.

\begin{figure}[tb]
\centering
\includegraphics[width=2.6in]{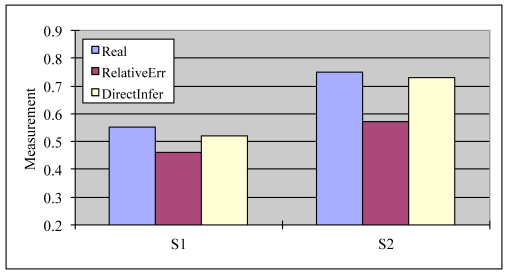}
\caption{Reported measurements on the positive population from different methods.}
\label{fig_real_camp}
\end{figure}

Our proposed approach was deployed into Turn's data management platform and generates weekly reports on the quality of our top data sources. We have received positive feedback from our campaign optimization managers in the field, commenting that the reported precisions are close to the real campaign results. Interestingly, by evaluating our data sources periodically, we are forming a positive reinforcement loop over their data quality: feeling the pressure, data providers work consistently to improve their data quality. For example, the estimated precision of one data source over a three month period is plotted in Figure~\ref{fig_trend}. It is clear that the data source's quality has been improved over this time period.

\begin{figure}[tb]
\centering
\includegraphics[width=2.6in]{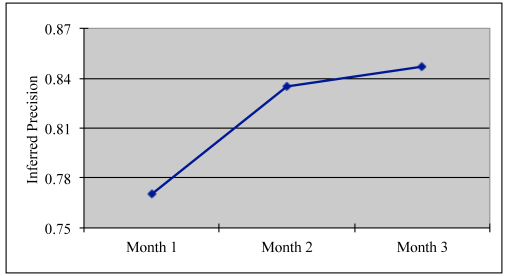}
\caption{Inferred precision of a data source over a three month period.}
\label{fig_trend}
\vspace{-6pt}
\end{figure}

\section{Conclusions and Future Work} \label{sec:conclusions}
In this paper, we have presented a novel framework to evaluate first or third-party data sources on user properties for online advertising, which is a particularly challenging task when the ground truth is reported in aggregate form. We call this problem \emph{data quality assessment}, and presented two solutions, one utilizing the data sources directly in a campaign, and another one, which utilizes outputs from multiple online advertising campaigns to optimize a set of probabilities which represent the ``goodness'' of a data source. We have also presented some use cases on how these evaluations can be utilized in online advertising domain, mainly in targeting, assessing the amount of money that an advertiser should pay for a data source, and forecasting. Some preliminary simulation and real-world results were also presented that show the effectiveness of our methodology, as well as some results on the performance of a well-established actual data provider for age categorization of users on multiple real-world advertising campaigns.

Possible future work mainly lies on the use cases of the evaluation output of our methodologies, as given in Section~\ref{sec:use_cases}. Our current focus is on accurate targeting of online users, which also needs to take into account the problem of combining multiple data sources and their quality assessments to come up with a better model.

\section*{Acknowledgments}
We thank many talented scientists and engineers at Turn for their help and feedback in this work.

\bibliographystyle{abbrv}
\bibliography{Geyik_TDQA}

\begin{thebibliography}{10}

\bibitem{adam_1989}
N.~R. Adam and J.~Wortmann.
\newblock Security-control methods for statistical databases: A comparative
  study.
\newblock {\em ACM Computing Surveys}, 21:515--556, 1989.

\bibitem{Casella02}
G.~Casella and R.~L. Berger.
\newblock {\em Statistical inference}, volume~2.
\newblock Duxbury Pacific Grove, CA, 2002.

\bibitem{chen_2006}
B.~Chen, L.~Chen, R.~Ramakrishnan, and D.~R. Musicant.
\newblock Learning from aggregate views.
\newblock In {\em Proc. IEEE ICDE}, 2006.

\bibitem{chen2010cheetah}
S.~Chen.
\newblock Cheetah: a high performance, custom data warehouse on top of
  mapreduce.
\newblock {\em Proc. VLDB Endowment}, 3(1-2):1459--1468, 2010.

\bibitem{cheplygina_2015}
V.~Cheplygina, D.~M.~J. Tax, and M.~Loog.
\newblock On classification with bags, groups and sets.
\newblock {\em Pattern Recognition Letters}, 59:11--17, 2015.

\bibitem{elmeleegy2013overview}
H.~Elmeleegy, Y.~Li, Y.~Qi, P.~Wilmot, M.~Wu, S.~Kolay, A.~Dasdan, and S.~Chen.
\newblock Overview of turn data management platform for digital advertising.
\newblock {\em Proc. VLDB Endowment}, 6(11):1138--1149, 2013.

\bibitem{ferri_2009}
C.~Ferri, J.~Hernandez-Orallo, and R.~Modroiu.
\newblock An experimental comparison of performance measures for
  classification.
\newblock {\em Pattern Recognition Letters}, 30:27--38, 2009.

\bibitem{gunawardana_2009}
A.~Gunawardana and G.~Shani.
\newblock A survey of accuracy evaluation metrics of recommendation tasks.
\newblock {\em Journal of Machine Learning Research}, 10:2935--2962, 2009.

\bibitem{jalali_2013}
A.~Jalali, S.~Kolay, P.~Foldes, and A.~Dasdan.
\newblock Scalable audience reach estimation in real-time online advertising.
\newblock In {\em Proc. ICDMW}, pages 629--637, 2013.

\bibitem{klee_2012}
K.-C. Lee, B.~Orten, A.~Dasdan, and W.~Li.
\newblock Estimating conversion rate in display advertising from past
  performance data.
\newblock In {\em Proc. ACM KDD}, pages 768--776, 2012.

\bibitem{musicant_2007}
D.~R. Musicant, J.~M. Christensen, and J.~F. Olson.
\newblock Supervised learning by training on aggregate outputs.
\newblock In {\em Proc. IEEE ICDM}, pages 252--261, 2007.

\bibitem{parker_2011}
C.~Parker.
\newblock An analysis of performance measures for binary classifiers.
\newblock In {\em Proc. IEEE ICDM}, pages 517--526, 2011.

\bibitem{hadoop_2012}
T.~White.
\newblock {\em Hadoop: The Definitive Guide}.
\newblock O'Reilly Media, Sebastopol, CA, 2012.

\bibitem{williams_2014}
M.~H. Williams, C.~Perlich, B.~Dalessandro, and F.~Provost.
\newblock Pleasing the advertising oracle: Probabilistic prediction from
  sampled, aggregated ground truth.
\newblock In {\em Proc. ACM ADKDD}, 2014.

\bibitem{yan_2009}
J.~Yan, N.~Liu, G.~Wang, W.~Zhang, Y.~Jiang, and Z.~Chen.
\newblock How much can behavioral targeting help online advertising?
\newblock In {\em Proc. ACM WWW}, pages 261--270, 2009.

\bibitem{yu_2015}
F.~X. Yu, K.~Choromanski, S.~Kumar, T.~Jebara, and S.~Chang.
\newblock On learning from label proportions.
\newblock {\em arXiv:1402.5902v2}, 2015.

\end{thebibliography}

\end{document}